\documentclass[11pt]{article}

\usepackage{amsmath, amsthm, amssymb, amscd, paralist}
\usepackage{graphicx, color}
\usepackage{bm}
\usepackage[title]{appendix}
\usepackage{hyperref}

\newcommand{\eps}{\varepsilon}

\newcommand{\argmin}{\mathop{\rm arg\min}}

\numberwithin{equation}{section}
\newtheorem{thm}{Theorem}
\newtheorem{lem}{Lemma}
\newtheorem{rem}{Remark}

\newtheorem{defi}{Definition}

\newcommand{\Mult}{\operatorname{Mult}}

\newcommand{\R}{\mathbb{R}}

\newcommand{\ba}{\mathbf{a}}

\newcommand{\mC}{\mathcal{C}}

\newcommand{\mF}{\mathcal{F}}

\newcommand{\mM}{\mathcal{M}}
\newcommand{\mN}{\mathcal{N}}

\renewcommand{\ba}{\mathbf{a}}
\newcommand{\bD}{\mathbf{D}}

\newcommand{\bp}{\mathbf{p}}
\newcommand{\bq}{\mathbf{q}}

\newcommand{\bu}{\mathbf{u}}
\newcommand{\bv}{\mathbf{v}}
\newcommand{\bw}{\mathbf{w}}
\newcommand{\bx}{\mathbf{x}}
\newcommand{\bX}{\mathbf{X}}
\newcommand{\by}{\mathbf{y}}

\newcommand{\bz}{\mathbf{z}}

\newcommand{\balpha}{\bm{\alpha}}

\newcommand{\bgamma}{\bm{\gamma}}

\newcommand{\bell}{\bm{\ell}}

\begin{document}

\title{Deep ReLU network approximation of functions on a manifold}

\author{Johannes Schmidt-Hieber\footnote{University of Twente, Drienerlolaan 5, 7522 NB Enschede, The Netherlands. \newline {\small {\em Email:} \texttt{a.j.schmidt-hieber@utwente.nl}} \newline The research has been supported by the Dutch STAR network and a Vidi grant from the Dutch science organization (NWO).}}

\date{}
\maketitle

\begin{abstract}
Whereas recovery of the manifold from data is a well-studied topic, approximation rates for functions defined on manifolds are less known. In this work, we study a regression problem with inputs on a $d^*$-dimensional manifold that is embedded into a space with potentially much larger ambient dimension. It is shown that sparsely connected deep ReLU networks can approximate a H\"older function with smoothness index $\beta$ up to error $\eps$ using of the order of $\eps^{-d^*/\beta}\log(1/\eps)$ many non-zero network parameters. As an  application, we derive statistical convergence rates for the estimator minimizing the empirical risk over all possible choices of bounded network parameters.  
\end{abstract}

%
%

\section{Introduction}

Suppose our training data are given by $(\bX_i, Y_i),$ $i=1,\ldots, n,$ with $\bX_i\in \R^d$ the $d$-dimensional input vectors and $Y_i \in \R$ the corresponding real-valued outputs. For many machine learning applications the inputs will lie in a somehow "small" subset compared to $\R^d.$ This is in particular true for image classification problems where the the input consists of vectors of pixel values. Images of cats and dogs for instance are a tiny subset of all images that can be created by arbitrarily assigning each pixel value. Although these subsets are small they cannot be parametrized. A natural way to model the input space, is to assume that the input vectors lie on an unknown $d^*$-dimensional manifold $\mM.$

The objective of deep learning is to reconstruct the relationship between input and output. If there is no noise in the observations, we can assume that there exists an unknown function $f$ with $Y_i=f(\bX_i)$ for all $i.$ Training a network on the dataset $(\bX_i,Y_i)_i$ should then return a deep neural network $g$ that is close to $f.$ In this framework we are not interested in the reconstruction of the manifold $\mM.$

A naive approach would be to approximate the function $f$ on the whole domain. It is, however, known that in order to achieve approximation error $\eps$ for a $\beta$-smooth H\"older function on $\mathbb{R}^d$ a ReLU network with depth $L$ will need $\eps^{-d/\beta}/(L \log(1/\eps))$ many non-zero parameters, see Lemma 1 in \cite{SH2017} for a precise statement. If the function is defined on a $d^*$-dimensional manifold $O(\eps^{-d^*/\beta})$ many non-zero network parameters should be, however, sufficient (up to logarithmic terms).  In this work we give a construction requiring $O(\eps^{-d^*/\beta} \log(1/\eps))$ non-zero parameters. Although the approximation rate is very natural, the proof is involved and requires a notion of smooth local coordinates that is of independent interest. 

The mathematical attempts to understand deep learning started only recently  and are still at their infancy. Nearly all results up to now require either strong assumptions or even avoid key aspects of deep networks such as depth and non-linearity in the parameter map by considering only shallow architectures and/or linearizations.

Deep learning can be viewed as a statistical method for prediction. To describe for which tasks this method performs well and when it fails is one of the key challenges that has to be answered by any sound theory of deep learning. An important aspect of this problem is to study the approximation theory induced by a deep network. To identify settings in which deep neural networks perform well, a good understanding of the approximation theory might be at least as useful as the analysis of algorithmic aspects. 

The most natural concept is to assume smoothness of the target function. Without any additional constraints, smoothness alone will lead to very slow approximation rates in the high dimensional setups for which deep learning still works well. It is therefore important to identify structural constraints for which fast approximation rates can be obtained.

In the statistics literature it has been argued that neural networks perform well if the function that needs to be learned has itself a composition structure, cf. \cite{horowitz2007, bauer2019,Poggio2017, mhaskar2016,SH2017}. Since a deep network can be viewed as a composition of simpler functions, this seems to be in a sense the natural structure that can be learned by this method. And indeed, many of the tasks for which deep neural networks are applied successfully have an underlying composition structure, including image classification, text analysis and game playing. For composition structures, optimal estimation rates can be attained by regressing a possibly large, but sparsely connected deep network to the data. It is also known that wavelet thresholding estimators can only obtain much slower convergence rates (\cite{SH2017}, Section 5). The theory is, however, far from being complete, as many unrealistic assumptions have to be imposed. 

In this work, we continue this line of work by studying an instance where the target function $f$ that has to be learned can be written for some (unknown) invertible map $\gamma$ as 
\begin{align}
	f= (f\circ \gamma^{-1}) \circ \gamma
	\label{eq.main_decompo}
\end{align}
and $f\circ \gamma^{-1}$ is somehow easier to approximate than $f.$ Since this representation has a composition structure, a deep network should be able to adapt to the structure and to learn $\gamma$ and $f\circ \gamma^{-1}$ in the first and last layers, respectively. Compared to a method that learns $f$ directly, faster approximation rates and consequently also faster statistical estimation rates should then be obtainable. In the case of function approximation on a $d^*$-dimensional manifold (ignoring for the moment that there are in general several charts), $\gamma$ is the local coordinate map and if $\gamma^{-1}$ is smooth, $f\circ \gamma^{-1}$ has the same smoothness as $f$ but is defined on $\R^{d^*}$ instead of $\R^d$ making $f\circ \gamma^{-1}$ easier to approximate due to the well known curse of dimensionality.

Another instance for a decomposition of the form \eqref{eq.main_decompo} is the Kolmogorov-Arnold representation. The KA representation has been viewed as a very specific neural network with two hidden layers. Although its connection to neural networks is still dubious, it is often listed among arguments why additional network layers are favorable. We will provide some additional insights in a companion article. 

Approximation theoretic results from the nineties mainly deal with shallow network and results are formulated to hold for large classes of activation functions, see \cite{mhaskar1996} for an example. The proofs typically exploit variations of the fact that if an activation function is smooth in a small neighborhood with non-vanishing derivatives, then all polynomials can be approximated arbitrarily well by a shallow network. Together with bounds on polynomial approximation this lead to approximation error estimates. One can then wonder why one should not work with polynomial approximations directly, see for instance \cite{pinkus1999}, p.177. 

For deep networks, the choice of the activation functions matters and the ReLU (rectified linear unit) activation function $\sigma(x)=\max(x,0)$ has been found to outperform other activation functions for classification tasks in terms of misclassification rates and the computational cost, cf. \cite{glorot2011}. From an approximation theoretic point of view, it makes therefore sense to study approximation rates for specific activation functions.

In this work, we specifically study ReLU networks and we heavily exploit the properties of the ReLU. One of the important features of the ReLU is the projection property $\sigma \circ \sigma=\sigma,$ which means that one can learn skip connections in the network. To illustrate the use of such a skip connection for approximation by deep networks assume that we need to construct several functions simultaneously within one network. In the case of manifold learning, this will be for instance the local coordinate maps and the functions defining a partition of unity. If for some of these functions we need networks with $L_1$ hidden layers and for the other functions $L_2<L_1$ hidden layers are required. By adding identity maps and using that $\sigma \circ \sigma=\sigma,$ we can then squeeze in additional hidden layers such that all functions can be simultaneously computed by a neural network with $L_1$ hidden layers. For more precise statements, see also \eqref{eq.add_layers} and \eqref{eq.parallelization}. One should also observe that for a general continuous activation function the approximation of the identity is difficult, see Proposition 2.9 in \cite{PetersenEtAl2018}. 

Moreover, we use other ReLU  specific network constructions in order to approximate for instance the multiplication of two inputs, see \cite{telgarsky2016, liang2016, Yarotsky2018}. Existing approximation theoretic results for general activation functions require that the size of the network parameters increases as the approximation error decreases. In practice, however, the network parameters are randomly initialized by small numbers and the trained network weights are typically close to the initialized ones and therefore do not become large. As we consider ReLU networks, we are able to show that good approximation rates are achievable even if all network parameters are bounded in absolute value by one. 

Literature on the related problem of reconstructing the manifold from samples includes \cite{MR2460286, MR2238670, MR2203281, 2016arXiv160204723B}. For function approximation on manifolds with networks, \cite{MR2647012} gives an approximation rate using so called Eignets and  \cite{chui2018} provides a survey of the field and proposes a variation of \eqref{eq.main_decompo} for function approximation on manifolds. 

For ReLU networks, approximation rates and statistical risk bounds have been obtained for multivariate function approximation under smoothnes constraints \cite{kohler2005, hamers2006, kohler2011, 2018arXiv181008033S, Hayakawa2019} and under structural constraints, including compositions of functions \cite{kohler2017, bauer2019} and piecewise smooth functions \cite{PETERSEN2018296, Imaizumi2019}. \cite{ECKLE2019232} compares deep ReLU networks and multivariate adaptive regression splines (MARS). While finishing the article, we became aware of the very recently released work by Nakada and Imaizumi \cite{2019arXiv190702177N}. In this article, approximation rates and statistical risk bounds are derived depending on the Minkowski dimension of the domain. While our approach is more inspired by the idea to rewrite the problem as a composition of function, Nakada and Imaizumi use a different proving strategy. In Remark \ref{rem.rem}, we describe an example where the Minkowski dimension is equal to the ambient dimension but still faster rates can be obtained using the approach described in this article. Another difference is that in our approach all network weights are bounded in absolute value by one, which, as mentioned above, is more in line with practice. 

The article is structured as follows. In Section \ref{sec.ReLU}, we define deep ReLU network function classes and recall important embedding properties. The network approximation of functions on manifolds is considered in Section \ref{sec.manifold}. This section also contains the definition of manifolds with smooth local coordinate maps and the main approximation error bound. An application to prediction error bounds for the empirical risk minimizer over sparsely connected deep ReLU networks can be found in Section \ref{sec.statistical_bound}. Longer proofs and additional technical lemmas are deferred to the appendix. 

{\it Notation:} For a vector $\bx=(x_1, \ldots, x_d),$ $|\bx|_p := (\sum_{i=1}^d |x_i|^p)^{1/p},$ $|\bx |_\infty := \max_i |x_i|,$ and $|\bx|_0 := \sum_i \mathbf{1}(x_i \neq 0).$ For a vector valued function $f$ defined on the domain $D,$ we set $\|f\|_{L^\infty(D)}:=\max_{\bx \in D}|f(\bx)|_\infty.$

\section{Deep feedforward neural networks}
\label{sec.ReLU}

Feedforward means that the information is passed in one direction through the network. We can either write a network function via a recursion or via, what is sometimes called, an unfolded representation. For our purposes the unfolded representation turns out to be more convenient and we follow the notation in \cite{SH2017}. Throughout the article, we work with ReLU networks, which means that the activation function $\sigma$ is taken to be $\sigma(x)=\max(x,0)=(x)_+.$ For vectors $\bv=(v_1, \ldots, v_r)^\top, \by=(y_1, \ldots, y_r)^\top\in \mathbb{R}^r,$ the shifted activation function is defined as $\sigma_{\bv} = (\sigma(y_1-v_1), \ldots, \sigma(y_r-v_r))^\top : \mathbb{R}^r \rightarrow \mathbb{R}^r.$ We also define a separate output activation function $\rho: \mathbb{R}^r \rightarrow \mathbb{R}^s$ that is chosen in dependence on the statistical problem. For regression, $\rho$ is the identity. For classification the softmax 
\begin{align*}
	\rho(\bz) = \Big(\frac{e^{z_1}}{\sum_{j=1}^r e^{z_j}}, \ldots, \frac{e^{z_r}}{\sum_{j=1}^r e^{z_j}} \Big)^\top
\end{align*}
is used mapping $\bz$ to a probability vector. The network architecture $(L, \bp)$ consists of a positive integer $L$ called the {\it number of hidden layers} or {\it depth} and a {\it width vector} $\bp=(p_0, \ldots, p_{L+1}) \in \mathbb{N}^{L+2}.$ A neural network with network architecture $(L, \bp)$ is then any function of the form 
\begin{align}
	f: \mathbb{R}^{p_0} \rightarrow \mathbb{R}^{p_{L+1}}, \quad \bx \mapsto f(\bx) = \rho W_L  \sigma_{\bv_L}   W_{L-1}  \sigma_{\bv_{L-1}}  \cdots  W_1 \sigma_{\bv_1}  W_0\bx,
	\label{eq.NN}
\end{align}
where $W_i$ is a $p_i \times p_{i+1}$ weight matrix and $\bv_i \in \mathbb{R}^{p_i}$ is a shift vector. Given a  function $g$ and a network architecture $(L, \bp),$ the approximation problem is to construct a network function of the form \eqref{eq.NN} with small approximation error. This means that $L, \bp$ are fixed and the adjustable parameters are the entries of the matrices $W_0, \ldots, W_L$ and the shift vectors $v_1, \ldots, v_L.$
 
Let $\|W_j\|_0$ denote the number of non-zero entries of $W_j$ and $\|W_j\|_\infty$ the maximum-entry norm of $W_j.$ The $s$-sparse networks with network parameters all bounded in absolute value by one are
\begin{align}
	\mF(L, \bp, s)
	&:= \Big\{f \, \text{of the form \eqref{eq.NN}} \, : \max_{j = 0, \ldots, L} \|W_j\|_{\infty} \vee |\bv_j|_\infty \leq 1, \,  \sum_{j=0}^L \|W_j\|_0 + |\bv_j|_0 \leq s\Big\},
	\label{eq.defi_bd_sparse_para_space}
\end{align}
with the convention that $\bv_0$ is a vector with coefficients all equal to zero. For fully connected networks, we  omit $s$ and write $\mF(L, \bp).$ As all the networks that we consider in this work have the same width for all hidden layers and the widths of the hidden layers are most of the time complicated expressions, it is convenient to introduce
\begin{align*}
	\mF(L, (p_0 \sim p \sim p_{L+1}), s) := \mF(L, (p_0,\underbrace{p, \ldots, p}_{L \ \text{times}},p_{L+1}, s).
\end{align*}

We frequently make use of the fact that for a fully connected network in $\mF(L, \bp),$ there are $\sum_{\ell=0}^L p_\ell p_{\ell+1}$ weight matrix parameters and $\sum_{\ell=1}^{L} p_\ell$ network parameters coming from the shift vectors. The total number of parameters is thus
\begin{align}
	\sum_{\ell=0}^L (p_\ell +1) p_{\ell+1} -p_{L+1}.
	\label{eq.nr_of_param_in_net}
\end{align}

To prove approximation error bounds, the general proof strategy is to build first smaller networks and then combine them into one big network. To combine networks, we make frequently use of the following rules. Firstly, network function spaces are enlarged by increasing the width vector and the number of non-zero network entries, 
\begin{align}
	\mF(L, \bp, s) \subseteq \mF(L, \bq, s') \quad \text{whenever} \ \bp \leq \bq \ \text{componentwise and} \ s\leq s'.
\end{align}
We can also compose two networks if the number of units in the output layer of the first network matches the number of units in the input layer of the second network. More concretely, for $\bp = (p_0, \ldots,  p_{L+1})$ and $\bp' = (p_0',  \ldots, p_{L+1}'),$ 
\begin{align}
\begin{split}
	f \in &\mF(L, \bp), \ g \in \mF(L',\bp'), \ p_{L+1} =p_0', \  \bv \in [-1,1]^{p_{L+1}}, \\ 
	&\Rightarrow \ \ g \circ \sigma_{\bv}(f) \in \mF(L+L'+1, (\bp, p_1', \ldots, p_{L'+1}')).
\end{split}	
	\label{eq.composition_general}
\end{align}
To synchronize the number of hidden layers for two networks, we can squeeze in add additional layers with identity weight matrix. Adding the extra hidden layers at the bottom of the network yields the inclusion
\begin{align}
	\mF(L, \bp,s) \subset \mF(L+q, (\underbrace{p_0,\ldots,p_0}_{q\text{ times}} , \bp), s+qp_0).
	\label{eq.add_layers}
\end{align}
Moreover, two networks of the same depth can be combined in order to compute two network functions in parallel,
\begin{align}
\begin{split}
	f \in &\mF(L, \bp ,s ), \ g \in \mF(L, \bp', s'), \ p_0=p_0' \\ 
	&\Rightarrow \ \  (f,g) \in \mF(L, (p_0, p_1+p_1', \ldots, p_{L+1}+p_{L+1}'), s+s').
\end{split}	
	\label{eq.parallelization}
\end{align}
Finally, for sparse networks having more than $s$ units in one hidden layer does not add anything to the function class and
\begin{align}
	\mF(L, \bp, s)=\mF\big(L, (p_0, p_1\wedge s, p_2\wedge s,  \ldots, p_L\wedge s, p_{L+1}), s\big).
	\label{eq.removal_nodes_identity}	
\end{align}
A proof of this fact is given in \cite{SH2017}.

\section{ReLU network approximation of a function on a manifold}
\label{sec.manifold}

As a prerequisite, we need to define H\"older functions on manifolds. Based on this definition, we can then introduce compact manifolds with H\"older smooth local coordinate charts and derive several properties such as existence of a partition of unity. The main approximation error bound is stated in Theorem \ref{thm.main_approx} at the end of the section.

{\bf H\"older functions and smoothness on a manifold:} For an index $\beta \leq 1,$ a function $f: D\rightarrow R^q,$ with $D$ an open set in $\R^r,$ has H\"older smoothness index $\beta,$ if $|f(\bx)-f(\by)|\leq C |\bx-\by|^\beta.$ Because of the equivalence of norms on finite dimensional vectors spaces, $|\cdot|$ can be any norm. H\"older continuity can be extended to $\beta >1.$ Let $\lfloor \beta \rfloor$ denote the largest integer strictly smaller than $\beta.$ For a real-valued function, the ball of $\beta$-H\"older functions with radius $K$ is then defined as
\begin{align*}
	\mC_r^\beta(D, K) = \Big\{ 
	&f:D \subset \R^r \rightarrow \R : \\
	&\sum_{\balpha : |\balpha| < \beta}\|\partial^{\balpha} f\|_{L^\infty(D)} + \sum_{\balpha : |\balpha |= \lfloor \beta \rfloor } \, \sup_{\stackrel{\bx, \by \in D}{\bx \neq \by}}
	\frac{|\partial^{\balpha} f(\bx) - \partial^{\balpha} f(\by)|}{|\bx-\by|_\infty^{\beta-\lfloor \beta \rfloor}} \leq K
	\Big\},
\end{align*}
where we used multi-index notation, that is, $\partial^{\balpha}= \partial^{\alpha_1}\ldots \partial^{\alpha_r}$ with $\balpha =(\alpha_1, \ldots, \alpha_r)\in \mathbb{N}^r$ and $|\balpha| :=|\balpha|_1.$ For a vector valued function $f,$ we write $f\in \mC_r^\beta(D, K)$ if all the component functions are in $\mC_r^\beta(D, K).$ This space is sometimes also denoted by $\mC^{\beta-\lfloor \beta\rfloor,\lfloor \beta\rfloor}$, cf. \cite{evans2010}.

For two vectors $\bx\in \R^r$ and $\balpha\in \mathbb{N}^r$ write $\bx^{\balpha}:= x_1^{\alpha_1}\cdot \ldots \cdot x_r^{\alpha_r}$ and $\balpha ! := \alpha_1 ! \cdot \ldots \cdot \alpha_r!.$ Define $P_{\ba}^{\beta} f(\bx) = \sum_{\alpha : |\alpha| < \beta} (\partial^{\balpha} f )(\ba)(\bx-\ba)^{\balpha}/\balpha ! .$ If $D$ is also convex and $\ba \in D,$ then, by Taylor's theorem for multivariate functions, there exists $\xi \in [0,1],$ such that
\begin{align*}
	f(\bx)= P_{\ba}^{\beta} f(\bx)
	+ 
	\sum_{\beta - 1 \leq |\balpha | < \beta } \Big((\partial^{\balpha} f )(\ba+\xi (\bx-\ba))
	- (\partial^{\balpha} f )(\ba) \Big)  \frac{(\bx-\ba)^{\balpha}}{\balpha !},
\end{align*}
and so, for $f\in \mC_r^\beta(D,K),$
\begin{align}
	\begin{split}
	\big | f(\bx) - P_{\ba}^{\beta} f(\bx)  \big|
	&=
	\sum_{\beta-1 \leq |\balpha |< \beta} \frac{|(\bx-\ba)^{\balpha} |}{\balpha !}  
	 \big| (\partial^{\balpha} f )(\ba+\xi (\bx-\ba)) - (\partial^{\balpha} f )(\ba) \big|  \\
	&\leq 
	K |\bx-\ba|_\infty^{\beta}.
	\end{split}
	\label{eq.Taylor_approx}
\end{align}
This means that a $\beta$-H\"older function can be approximated in the neighborhood of any point $\ba$ by a polynomial of order $\lfloor \beta \rfloor$ up to an approximation error $O(|\bx-\ba|_\infty^{\beta}).$ Via this property, we can define H\"older functions on any subset of a metric space. To denote the difference, we use $C$ instead of the calligraphic $\mC$ and write
\begin{align*}
	C_r^\beta(D, K) := \Big\{ 
	&f:D \subset \R^r \rightarrow \R : \ \text{for any} \ \ba \in D, \ \exists  \ P_{\ba}^{\beta} f(\bx) = \sum_{0 \leq |\bgamma| < \beta} c_{\bgamma, \ba} \bx^{\bgamma}  \\ &\text{with} \ \big | f(\bx) - P_{\ba}^{\beta} f(\bx)  \big|\leq K |\bx-\ba|_\infty^{\beta}, \ \forall \bx \in D,  \ \text{and} \ \sup_{\ba \in D, \ 0\leq \bgamma < \beta} \bgamma ! |c_{\bgamma, \ba}|\leq K 
	\Big\}.
\end{align*}
As before, for vector valued functions, $f\in C_r^\beta(D, K)$ means that all component functions are in this space. We also define $$C_r^\beta(D) := \cup_{K>0} C_r^\beta(D,K).$$ If $D$ is a bounded domain, it is not hard to show that $C_r^\beta(D)\subseteq C_r^{\beta'}(D)$ whenever $\beta'\leq \beta.$ If $f\in C_r^1(D,K),$ then, $P_{\ba}^{\beta} f =f(\ba)$ and thus $f$ is Lipschitz with Lipschitz constant bounded by $K.$ Together with the embedding property, this shows that $f\in C_r^\beta(D)$ is Lipschitz if $\beta\geq 1$ and $D$ is bounded. 

\begin{lem}
\label{lem.smoothness_implies_Hoelder}
Let $D\subseteq \R^r$ be an open set and consider a bounded function $f: D \rightarrow \R.$ Suppose that all partial derivatives of $f$ exist, are bounded and vanish outside of a bounded set. Then, $f \in C_r^\beta(D)$ for all $\beta>0.$
\end{lem}

\begin{proof}
Rewriting
\begin{align*}
	P_{\ba}^{\beta} f(\bx) 
	= \sum_{\alpha : |\alpha| < \beta} \frac{(\partial^{\balpha} f )(\ba)(\bx-\ba)^{\balpha}}{\balpha !}
	=
	\sum_{0 \leq |\bgamma| < \beta} \bx^{\bgamma} \sum_{\bgamma \leq \balpha \&  |\balpha|< \beta} (\partial^{\balpha} f) (\ba) \frac{(- \ba)^{\balpha-\bgamma}}{\bgamma ! (\balpha -\bgamma)!},
\end{align*}
the result follows from \eqref{eq.Taylor_approx}. 
\end{proof}


\begin{lem}
\label{lem.composition}
Let $D\subset \R^r$ be a bounded set and $\beta\geq 1.$ If $f \in C_r^\beta(D)$ is a function mapping to $\R^q$ and $g\in C_q^\beta(\R^q),$ then, $g\circ f \in C_r^\beta(D).$ 
\end{lem}

The proof is given in the appendix.

{\bf Manifolds with smooth local coordinates:} We consider a $d^*$-dimensional compact manifold $\mM.$ By definition of a compact manifold, there exist open sets $V_1, \ldots, V_r\subset \mM$ with $\cup_j V_j =\mM$ and coordinate maps $\psi_j: \R^{d^*} \rightarrow V_j,$ $j=1, \ldots, r.$ A smooth manifold only guarantees that the transfer functions are smooth. It does not say anything about the smoothness of a local coordinate map or its inverse. We will therefore impose additional structure here.

\begin{defi}
\label{defi.smooth_maps}
We say that a compact $d^*$-dimensional manifold $\mM\subset \R^d$ has smooth local coordinates if there exist charts $(V_1,\psi_1), \ldots,(V_r,\psi_r),$ such that for any $\gamma>0,$ $\psi_j \in C_d^\gamma(V_j)$ and $\psi_j^{-1}\in C_{d^*}^\gamma(\psi_j(V_j))$ for all $j=1, \ldots, r.$
\end{defi}

As an example consider the unit sphere $S^1 \subset \R^2.$ Every point on the sphere can be written as $(\sin t, \cos t),$ $0\leq t < 2\pi.$ Defining the charts such that $(x,y) \mapsto \arcsin(x)$ and $(x,y) \mapsto \arccos(y)$ are invertible and smooth on the covering, it can be shown that $S^1$ is a manifold with smooth local coordinates in the sense of Definition \ref{defi.smooth_maps}.

The notion of smooth local coordinates is also compatible with the standard operations to construct more complicated manifolds from simpler ones. If, for instance, $\mM$ and $\mN$ are compact manifolds with dimension $d_1^*$ and $d_2^*,$ respectively and smooth local coordinates, then, $\mM \times \mN$  is a compact $d_1^*+d_2^*$-dimensional manifold with smooth local coordinates. This can be checked as follows. If  $(V_1,\psi_1), \ldots,(V_r,\psi_r)$ are the charts for $\mM$ and  $(W_1,\phi_1), \ldots,(W_r,\phi_s)$ the charts for $\mN,$ then, one can directly verify the conditions for the charts $(V_k \times W_\ell, (\bx, \by) \mapsto (\psi_k(\bx), \phi_\ell(\by))^\top),$ $k=1, \ldots, r, $ $\ell=1, \ldots, s.$

{\bf Partition of unity:} On a compact manifold $\mM,$ it is always possible to find a finite partition of unity, see Section 13.3 in \cite{Tu2011}. As we require additional smoothness and support properties, a more refined result is needed. This leads to several complications in the proof which is deferred to the appendix.

For a vector $\bx$ and a set $A$ on the same vector space, $|\bx-A|_\infty:=\inf_{\ba \in A} |\bx-\ba|_\infty.$ For $\delta >0,$ define
\begin{align*}
	V_j^{-\delta} :=\{\by \in V_j: |\by -(\mM\setminus V_j)|_\infty \geq \delta\}.
\end{align*}

\begin{lem}
\label{lem.partition_of_unity}
Consider a compact $d^*$-dimensional manifold $\mM\subset \R^d$ with smooth local coordinates. Then, there exist a $\underline \delta >0$ and non-negative functions $\tau_j:\mM \rightarrow \R,$ $j=1, \ldots,r,$ such that for any $\gamma>0,$ and any  $\bx \in \mM,$ $\{\by\in \mM: \tau_j(\by)>0\}\subseteq V_j^{-\underline{\delta}},$  $\tau_j \in C_d^\gamma(\mM),$ and $\sum_{j=1}^r \tau_j(\bx) =1.$
\end{lem}

{\bf Main idea for construction of network approximation:} The strategy is to build a deep neural network with good approximation properties combining simpler networks approximating $\psi_j,$ $f\circ \psi_j^{-1},$ and $\tau_j,$ $j=1,\ldots,r.$ Combining the single networks we can then construct a deep ReLU network mimicking the left hand side in the identity
\begin{align}
	\sum_{j=1}^K \tau_j(\bx) \big( f\circ \psi_j^{-1}\big) \circ \psi_j(\bx) = f(\bx), \quad \text{for all} \ \bx \in \mM
	\label{eq.f_split}
\end{align}
where each summand is defined as zero if $\tau_j(\bx)=0.$ The main steps of the network construction are also summarized in Figure \ref{fig.manifold}.

\begin{figure}
\begin{center}
	\includegraphics[scale=0.6]{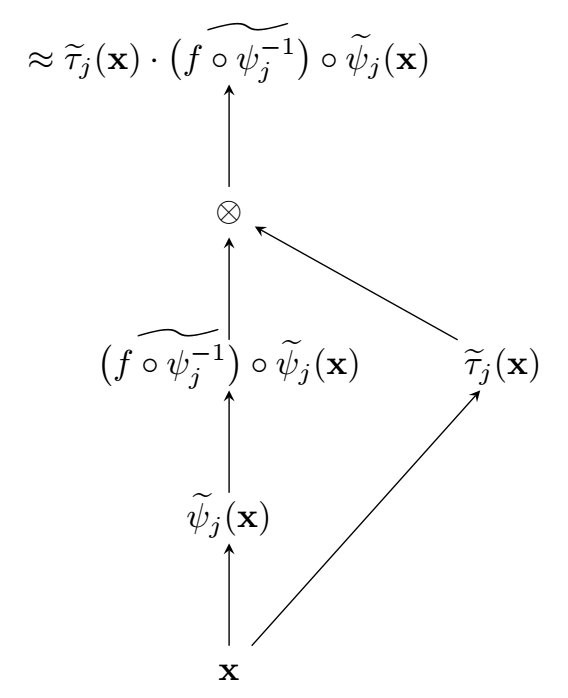} 
\end{center}
\vspace{-0.5cm}
\caption{\label{fig.manifold} Combining smaller networks to build an approximation of the functions $\widetilde \tau_j(\bx) \cdot \big( \widetilde {f\circ \psi_j^{-1}\big)} \circ \widetilde \psi_j(\bx)$. Here, $\otimes$ stands for the multiplication network introduced in Lemma \ref{lem.mult}.}
\end{figure}

{\bf Approximations of H\"older functions by deep ReLU networks:} 

\begin{thm}
\label{thm.approx_network_one_fct_neighborhood}
Suppose that $U\subset \R^d$ is a bounded set. For any function $g:U \rightarrow \R^{d'}$ in the H\"older space $g\in \mC_d^\beta(U, K)$ and any integers $m \geq 1$ and $N \geq  5^d\vee (\beta+1)^d \vee (K+1)e^d,$ there exists a network $$\widetilde g \in \mF\big(L, \big(d \sim  6d'(d+\lceil \beta\rceil)N \sim d'\big), s\big)$$ with depth $$L=9+(m+5)(1+\lceil \log_2 (d\vee \beta) \rceil)$$ and number of parameters
\begin{align*}
	s\leq 142 d' (d+\beta+1)^{3+d} N (m+6),
\end{align*}
such that with $R:= 1\vee 4\max_{\bx \in U}|\bx|_\infty,$
\begin{align*}
	\| \widetilde g - g\|_{L^\infty(U)}\leq  (2KR^\beta+1)(1+d^2+\beta^2) 6^d N2^{-m}+ K(9R)^\beta N^{-\frac{\beta}d}.
\end{align*}
\end{thm}

The upper bound for the approximation error consists of two terms. Since we have $s \asymp Nm$ parameters we expect the well-known approximation rate $(Nm)^{-\beta/r}.$ The second term coincides with this rate up to the factor $m.$ This suggests to choose $m$ small. Then, however, the first term in the bound of the approximation error will become large. The optimal trade-off is $(1+\beta/d)\log_2(N)\leq m \lesssim \log(N)$ in which case approximation error $O(N^{-\frac{\beta}d})$ can be achieved for $O(N\log N)$ non-zero network parameters. 

The proof of Theorem \ref{thm.approx_network_one_fct_neighborhood} builds on Theorem 5 in \cite{SH2017} constructing a deep ReLU network $\widetilde f$ from simpler networks that compute a linear combination of local Taylor approximations of $f.$ More precisely, $U$ is rescaled to fit into the hypercube $[0,1]^d$ and then a uniform grid on this hypercube is constructed. On each of the grid points, we construct a sub-network  approximating a $\lfloor \beta\rfloor$-th order Taylor approximation at that point, where as before, $\lfloor \beta\rfloor$ denotes the largest integer strictly smaller than $\beta.$ The function reconstruction at a specific point is then a weighted sum of the Taylor approximation at the surrounding points. There are several technical issues that occur, for instance, for a point close to the boundary, some of the surrounding grid points lie outside the rescaled version of the set $U$ and the $\lfloor \beta\rfloor$-th order Taylor approximation on these points is not well defined.

{\bf Main approximation error bound:}  Following the strategy outlined in Figure \ref{fig.manifold} to construct a network mimicking the left hand side in \eqref{eq.f_split}, we are now able to state the main result of the article.

\begin{thm}
\label{thm.main_approx}
Let $\mM\subset \R^d$ be a compact $d^*$-dimensional manifold with smooth local coordinates. Then, there exist positive constants $c, C, C',$ such that for any $0< \eta \leq 1/2,$ any $c\log(1/\eta) \leq L,$ any $p> C\eta^{-d^*/\beta}$ and any $s\geq C' L\eta^{-d^*/\beta},$ 
\begin{align*}
	\sup_{f: \R^d\rightarrow [-1,1], \ f\in \mC_d^\beta(\R^d, K)} \, \inf_{\widetilde f\in \mF(L, (d \sim p\sim  1), s)} \, \| \widetilde f - f\|_{L^\infty(\mM)}\leq \eta.
\end{align*}
\end{thm}

The proof allows to exactly quantify the dependence of the constants $c, C,$ and  $C',$  on $d,d^*,\beta, K.$ This, however, leads to complicated expressions. 

\begin{rem}
\label{rem.rem}
The approach using \eqref{eq.f_split} can be easily extended if there is additional invariance in the function $f.$ To illustrate this, suppose that $\mM\subset \R^{d_1}$ is a compact $d^*$-dimensional manifold with smooth local coordinates in the sense of Definition \ref{defi.smooth_maps}. Denote the charts of $\mM$ by $(V_1,\psi_1), \ldots, (V_r, \psi_r)$ and the functions forming a partition of unity by $\tau_1,\ldots, \tau_r.$ Let $U\subset \R^{d_2}$ be a compact set containing the vector $(0,\ldots,0)^\top.$ Suppose now that we want to approximate $f:\R^d\rightarrow \R$ with $d=d_1+d_2$ on the set $\mM\times U.$ Assume moreover that $f$ is independent of the last components in the vector in the sense that $f(x_1,\ldots,x_d)=f(x_1,\ldots,x_{d_1}).$ Defining $\overline \psi_j(x_1,\ldots,x_d) :=\psi(x_1,\ldots,x_{d_1})$ and $\overline \tau_j(x_1,\ldots,x_d) :=\tau(x_1,\ldots,x_{d_1}),$ we now have 
\begin{align*}
		\sum_{j=1}^K\overline\tau_j(\bx) \big( f\circ \psi_j^{-1}\big) \circ \overline \psi_j(\bx) = f(\bx), \quad \text{for all} \ \bx \in \mM\times U.	
\end{align*}
If $f\in C_d^\beta(\R^d),$ it is clear that $O(L\eta^{-d^*/\beta})$ network parameters are needed to approximate $f$ on $\mM\times U$ up to sup-norm error $\eta.$ To obtain this rate, we only need to know $d^*$ and  $\beta$ but not what the invariance property of $f$ is. One should also observe that the Minkowski dimension of the set $\mM\times U$ is $d$ if $U$ is for instance a ball. This means that if the approach in \cite{2019arXiv190702177N} is followed without modification, $O(L\eta^{-d/\beta})$ many non-zero network parameters are needed, which can be considerably  larger if $d\gg d^*.$
\end{rem}

\section{Statistical risk bounds}
\label{sec.statistical_bound}

In this section, we convert the approximation error bounds in statistical risk estimates. Suppose we observe $n$ identically distributed pairs $(\bX_i, Y_i) \in \R^d\times \R$ with 
\begin{align}
	Y_i = f_0(\bX_i) + \eps_i, \quad i=1, \ldots, n,
	\label{eq.mod}
\end{align}
and $(\eps_i)_i$ a sequence of i.i.d. standard normal measurement errors that are independent of the design vectors $\bX_i.$ As before, the $\bX_i$ are assumed to lie on a compact $d^*$-dimensional manifold $\mM$ with smooth local coordinates. The manifold $\mM$ is unknown and we also suppose that $\mM \subseteq [0,1]^d.$ We moreover assume that $f_0\in C_d^\beta(\R^d)$ for some $\beta$ and $\|f_0\|_{L^\infty(\R^d)}\leq 1.$

In statistics, a lot of research has been devoted to approximation bounds and statistical risk bound under shape constraints on the regression function $f_0.$ But relatively little is known for constraints on the design. An exception is \cite{bickel2007}, studying locally polynomial estimators in nonparametric regression with design on an unknown manifold.

The prediction risk is the expected loss that we suffer by predicting the output for a new input vector that is generated from the same distribution as the design vectors in the training set. Thus, with $\bX \stackrel{\mathcal{D}}{=}\bX_1$ being independent of the sample $(\bX_i,Y_i)_i,$ the prediction risk is given by 
\begin{align*}
	R(\widehat f_n, f_0) := E_{f_0}\big[\big( \widehat f_n(\bX) - f_0(\bX) \big)^2\big],
\end{align*}
where $E_{f_0}$ denotes the expectation over $\bX$ and independent $(\bX_i, Y_i)_{i=1,\ldots,n}$ generated from model \eqref{eq.mod}.

We study the empirical risk estimator (ERM) over the class $\mF(L,\bp, s) \cap \{f:[0,1]^d \rightarrow [-1,1]\},$ that is, 
\begin{align}
	\widehat f_n \in \argmin_{f\in \mF(L,\bp, s) \cap \{f:[0,1]^d \rightarrow [-1,1]\}} \frac 1n \sum_{i=1}^n \big(Y_i - f(\bX_i)\big)^2.
	\label{eq.ERM}
\end{align}
To compute the ERM is extremely hard if not infeasible for non-convex function spaces such as neural networks. On the contrary, via empirical processes, theoretical guarantees can be obtained. Therefore, we restrict ourselves here to the ERM analysis and refer to \cite{SH2017} for an extension. By Lemma 4, Lemma 5 in \cite{SH2017} and \eqref{eq.removal_nodes_identity}, we have that the prediction risk of the empirical risk minimizer is bounded by
\begin{align}
	R( \widehat f_n, f_0)
	&\leq C \Big[ \inf_{f\in \mF} E\big[\big(f(\bX) - f_0(\bX) \big)^2 \big] 
	+ \frac{ (s+1) \log ( 4n (L+1) (s+1)^L (d+1)) + 1}{n} \Big],
	\label{eq.prediction_error_bound}
\end{align}
with $C$ a universal constant, $\mF= \mF(L,\bp, s) \cap \{f:[0,1]^d \rightarrow [-1,1]\}$ and $E$ the expectation taken over $\bX.$ Inequalities of this form are also called oracle inequalities as the statistical risk of the estimator is bounded by the best risk of any element in the class plus some extra term that penalizes the model complexity. 
The next result is now a straightforward consequence of the abstract risk bound for the empirical risk minimizer in \eqref{eq.prediction_error_bound} and Theorem \ref{thm.main_approx}.

\begin{thm}
\label{thm.main}
Consider model \eqref{eq.mod} with design on a $d^*$-dimensional compact manifold $\mM\subset [0,1]^d$ and regression function $f: \R^d \rightarrow [-1,1]$ in $C_d^\beta(\R^d,K).$ Let $\widehat f_n$ be the ERM in \eqref{eq.ERM}. Then, there exist constants $c, Q>0$ that are independent of $n,$ such that if 
\begin{compactitem}
\item[(i)] $c\log n \leq L,$
\item[(iii)] $n^{d^*/(2\beta+d^*)} \lesssim p,$
\item[(iv)] $s \asymp  L n^{d^*/(2\beta+d^*)},$
\end{compactitem}
then,
\begin{align}
	R(\widehat f_n, f_0) \leq Q n^{-\frac{2\beta}{2\beta+d^*}} L^2 \log(n).
	\label{eq.main1}
\end{align}
\end{thm}

In particular, $L$ can be chosen of order $\log(n)$ such that the risk is bounded by $$ \text{constant}\times n^{-\frac{2\beta}{2\beta+d^*}} \log^3(n).$$

\begin{proof}
Choose $\eta=R n^{-\frac{2\beta}{2\beta+d^*}}$ for a sufficiently large constant $R,$ such that, if $C$ and $C'$ are as in Theorem \ref{thm.main_approx}, $p> C\eta^{-d^*/\beta}$ and $s\geq C' L\eta^{-d^*/\beta}.$ The result follows from Theorem \ref{thm.main_approx} and \eqref{eq.prediction_error_bound}. 
\end{proof}

\section{Proofs}

\begin{proof}[Proof of Lemma \ref{lem.composition}]
It is enough to consider the case that $g: \R^q \rightarrow \R.$ By definition of $f \in C_r^\beta(D),$ there exists an $q$-dimensional vector $P_{\ba}^{\beta} f$ containing polynomials of degree $\lfloor \beta \rfloor,$ and a positive $K,$ such that  $| f(\bx) - P_{\ba}^{\beta} f(\bx) |_\infty \leq K |\bx-\ba|_\infty^{\beta}, \ \forall \bx, \ba \in D.$ The coefficients of $P_{\ba}^{\beta} f$ are uniformly bounded over $\ba\in D.$ In the same way $g\in C_q^\beta(\R^q)$ implies existence of a polynomial  $Q_{\bu}^{\beta} g$ with uniformly bounded coefficients over $\bu \in \R^q$ approximating $g(\by)$ up to an error $K |\by-\bu|_\infty^{\beta}$ for all $\by, \bu \in \R^q.$

The composition $R_{\ba}^\beta(\bx):=Q_{f(\ba)}^{\beta} g(P_{\ba}^{\beta} f(\bx))$ is a polynomial in $\ba-\bx$ of degree $2\lfloor \beta \rfloor.$ Denote by $\widetilde R_{\ba}^\beta(\bx)$ the polynomial obtained by removing in $R_{\ba}^\beta(\bx)$ all terms $(\bx-\ba)^{\bgamma}$ with degree $|\bgamma| \geq \beta.$ Consequently, $\widetilde R_{\ba}^\beta(\bx)$ is a polynomial of degree $\lfloor \beta \rfloor.$ 

Th coefficient of $\widetilde R_{\ba}^\beta(\bx)$ are uniformly bounded over $\ba.$ It remains to show that $\widetilde R_{\ba}^\beta(\bx)$ approximates $f\circ g$ up to an error constant$\times |\bx-\ba|_\infty^{\beta}.$ Since $D$ is a bounded domain, we conclude that $|\widetilde R_{\ba}^\beta(\bx) - R_{\ba}^\beta(\bx)|\leq K' |\bx-\ba|_\infty^{\beta}$ for some finite constant $K'.$ We must have that $B:=\cup_{\ba} P_{\ba}^{\beta} f(D) \cup f(D)$ is a bounded subset of $\R^q.$ Moreover, since $Q_{f(\ba)}^{\beta} g$ is a polynomial, it is Lipschitz on $B$ and a Lipschitz constant can be chosen independent of $\ba \in D.$ Similarly, one can show that for sufficiently large $K'',$ $\max_{\ba \in D}|P_{\ba}^{\beta} f(\bx)-P_{\ba}^{\beta} f(\by)|\leq K''|\bx-\by|_\infty$ for all $\bx, \by\in D.$ Because of $\beta \geq 1$ also $f$ is Lipschitz. Thus, there exists a constant $C$ such that
\begin{align*}
	\big| R_{\ba}^\beta(\bx) - g\circ f (\bx) \big|
	&\leq  \big| Q_{f(\ba)}^{\beta} g(P_{\ba}^{\beta} f(\bx)) - Q_{f(\ba)}^{\beta} g(f(\bx)) \big|
	+ \big |  Q_{f(\ba)}^{\beta} g(f(\bx)) -  g\circ f (\bx) \big| \\
	&\leq C \big| P_{\ba}^{\beta} f(\bx) - f(\bx)\big|_\infty 
	+ C \big| f(\ba) - f(\bx) \big|_\infty^\beta \\
	&\leq C K |\ba - \bx|_\infty^\beta + C|\ba-\bx|_\infty^\beta,
\end{align*}
completing the proof for (i).
\end{proof}

\begin{proof}[Proof of Lemma \ref{lem.partition_of_unity}]
Since $\mM\setminus V_j$ is a closed set, we have that $\delta(\bx):=|\bx-(\mM\setminus V_j)|_\infty := \inf_{\by \in \mM\setminus V_j} |\bx-\by|_\infty >0$ implying that $B_{\delta(\bx)/2}^\infty(\bx) \cap \mM \subset V_j$ where $B_{\eps}^\infty(\bx) = \{\by :|\by-\bx|_\infty \leq \eps\}$ denotes the $|\cdot|_\infty$-norm ball around $\bx$ with radius $\eps.$

Thus, for any $j=1,\ldots,r$ and any $\bx \in V_j$ it is possible to construct a smooth and non-negative function $\lambda_{\bx}:\R^{d^*} \rightarrow [0,\infty)$ such that $\psi_j(\bx) \in \{\by \in \R^{d^*}: \lambda_{\bx}(\by)>0\}\subseteq \psi_j(B_{\delta(\bx)/2}^\infty(\bx)\cap \mM)$ and for any $\beta>0,$ $\lambda_{\bx} \in C_{d^*}^\beta(\R^{d^*}).$ By construction
\begin{align*}
	\bigcup_{j=1, \ldots,r} \, \bigcup_{\bx\in V_j}\, \psi_j^{-1}\Big(\big\{\by \in \R^{d^*}: \lambda_{\bx}(\by)>0\big\}\Big) =\mM.
\end{align*}
This is a union over open sets and since $\mM$ is compact, we can select points $\bx_\ell \in \psi_{s(\ell)}(V_{s(\ell)})$ for $\ell=1,\ldots,m$ and some finite $m,$ where $s:\{1,\ldots,m\}\rightarrow \{1,\ldots, r\}$ maps the points $\bx_\ell$ to the corresponding charts, such that 
\begin{align*}
	\bigcup_{\ell=1, \ldots,m} \, \psi_{s(\ell)}^{-1}\Big(\big\{\by \in \R^{d^*}: \lambda_{\bx_\ell}(\by)>0\big\}\Big) =\mM.
\end{align*}
For $j=1, \ldots, r,$ define $\nu_j = \sum_{\ell \in s^{-1}(j)} \lambda_{\bx_\ell}.$ Then, $\cup_{j=1,\ldots,r} \psi_j^{-1}(\{\bx \in \R^{d^*}: \nu_j(\bx)>0\}) =\mM.$ In a next step, define $\sigma_j :\mM \rightarrow [0,\infty)$ via
\begin{align*}
\sigma_j(\bx) = 
\begin{cases}
\nu_j \circ \psi_j (\bx), &\text{if} \ \bx\in V_j, \\
0 &\text{otherwise}.
\end{cases}
\end{align*}

By definition of the maps $\lambda_{\bx},$ we have that 
\begin{align*}
	\inf \big\{ |\bu-\bv|_\infty : \bu \in \{\by:\sigma_j(\by)>0\}, \bv\in \mM\setminus V_j \big \} \geq \min_{\ell=1,\ldots, m} \frac{\delta(\bx_{\ell})}{2} =: \underline{\delta} >0.
\end{align*}
By Lemma \ref{lem.composition}, we have that for any $\beta \geq 1,$ $\sigma_j \in C_d^\beta(V_j).$ We now show that this can be extended such that for any $\beta \geq 1,$ $\sigma_j \in C_d^\beta(\mM).$ The property $\sigma_j \in C_d^\beta(V_j)$ ensures existence of a polynomial $P_{\ba}^\beta\sigma_j$ of degree $\lfloor \beta\rfloor$ with bounded coefficients where the bound is independent of $\ba.$
Choose 
\begin{align*}
Q_{\ba}^\beta\sigma_j 
:=
\begin{cases}
P_{\ba}^\beta\sigma_j, &\text{if} \ |\ba- (\mM\setminus V_j)|_\infty > \underline{\delta}/2, \\
0  &\text{otherwise}.
\end{cases}
\end{align*}
Obviously, also for the polynomials $Q_{\ba}^\beta\sigma_j,$ all coefficients can be uniformly bounded over $\ba.$

To show that $|Q_{\ba}^\beta\sigma_j(\bx)-\sigma_j(\bx)|/|\ba -\bx|_\infty^\beta$ is bounded, we have to consider several cases. Assume first that $\bx \in V_j$ and $|\ba- (\mM\setminus V_j)|_\infty > \underline{\delta}/2.$ Since $\sigma_j \in C_d^\beta(V_j),$ $|Q_{\ba}^\beta\sigma_j(\bx)-\sigma_j(\bx)|/|\ba -\bx|_\infty^\beta$ is bounded. If $\bx \in \mM\setminus V_j,$ and  $|\ba- (\mM\setminus V_j)|_\infty > \underline{\delta}/2,$ then, $\underline{\delta}/2\leq |\bx-\ba|_\infty,$ and
\begin{align*}
	\big| Q_{\ba}^\beta\sigma_j(\bx)-\sigma_j(\bx) \big|
	&= \big| Q_{\ba}^\beta\sigma_j(\bx)\big|
	\leq \sup_{\bu \in V_j} \big\|Q_{\bu}^\beta\sigma_j \big\|_{L^\infty(\mM)}
	\Big(\frac{2}{\underline{\delta}}\Big)^{\beta} |\bx-\ba|_\infty^\beta.
\end{align*}
If now $|\ba- (\mM\setminus V_j)|_\infty \leq \underline{\delta}/2$ and $|\bx- (\mM\setminus V_j)|_\infty<\underline\delta,$ then, $|Q_{\ba}^\beta\sigma_j(\bx)-\sigma_j(\bx)|=|0-0|=0.$ Finally for the case $|\ba- (\mM\setminus V_j)|_\infty \leq \underline{\delta}/2$ and $|\bx- (\mM\setminus V_j)|_\infty \geq \underline\delta,$ it follows that $|\bx-\ba|\geq \underline \delta/2$ and $|Q_{\ba}^\beta\sigma_j(\bx)-\sigma_j(\bx)|\leq \|\sigma_j\|_{L^\infty(\mM)}\leq \|\sigma_j\|_{L^\infty(\mM)} (2/\underline{\delta})^{\beta} |\bx-\ba|_\infty^\beta.$ This shows that there exists a constant $K'$ such that for all $\ba, \bx\in \mM,$ $|Q_{\ba}^\beta\sigma_j(\bx)-\sigma_j(\bx)|\leq K' |\ba -\bx|_\infty^\beta,$ implying that also $\sigma_j \in C_d^\beta(\mM).$

For $\bx \in V_j,$ we have $\sigma_j(\bx) = \nu_j \circ \psi_j(\bx)$ and therefore $\sigma_j(\bx)>0$ if $\bx \in \psi_j^{-1}( \{\by \in \R^{d^*}: \nu_j(\by)>0\}).$ Hence, $\sum_{j=1}^r \sigma_j(\bx)>0$ for all $\bx\in \mM.$ Since $\bx \mapsto \sum_{j=1}^r \sigma_j(\bx)$ is continuous and $\mM$ is compact, also $0<\underline{\sigma}:=\inf_{\bx\in \mM}\sum_{j=1}^r \sigma_j(\bx)\leq \sup_{\bx\in \mM}\sum_{j=1}^r \sigma_j(\bx) =: \overline{\sigma} < \infty.$

Choose now $G:\R^2\rightarrow \R$ such that $G(u,v)=u/v$ for all $\underline{\sigma}\leq u,v\leq \overline{\sigma},$ $G$ vanishes outside a bounded set and all partial derivatives of $G$ exist. This can be achieved for instance by choosing a smooth function $K$ with $K(x)=0$ for all $x\leq 1/4$ and $K(x)=1$ for all $x\geq 3/4$ and defining $G(u,v):=K(u \underline{\sigma})K(v \underline{\sigma})K(\overline \sigma +1 - u)K(\overline \sigma +1 - v)  u/v.$ By Lemma \ref{lem.smoothness_implies_Hoelder}, it then follows that $G\in C_2^\beta(\R^2)$ for all $\beta>0.$ Since $(\sigma_j, \sum_{\ell=1}^r \sigma_\ell) \in C_d^\beta(\mM),$ we can conclude by Lemma \ref{lem.composition} that $\tau_j :=G(\sigma_j, \sum_{\ell=1}^r \sigma_\ell) \in C_d^\beta(\mM).$ This completes the proof.
\end{proof}

\begin{proof}[Proof of Theorem \ref{thm.approx_network_one_fct_neighborhood}]
Using the parallelization property \eqref{eq.parallelization}, it is enough to show the result for $d'=1.$ In this case, the statement is a modification of Theorem 5 in \cite{SH2017}. The theorem states that if for any function $f\in \mC_d^\beta([0,1]^d, K)$ and any integers $m \geq 1$ and $N \geq  (\beta+1)^d \vee (K+1)e^d,$ there exists a network $\widetilde f \in \mF\big(L ,\big(d \sim  6(d+\lceil \beta\rceil)N \sim 1\big), s\big)$ with depth $$L=8+(m+5)(1+\lceil \log_2 (d\vee \beta) \rceil)$$ and number of parameters $s\leq 141 (d+\beta+1)^{3+d} N (m+6),$ such that 
\begin{align*}
	\| \widetilde f - f\|_{L^\infty([0,1]^d)}\leq  (2K+1)(1+d^2+\beta^2) 6^d N2^{-m}+ K3^\beta N^{-\frac{\beta}d}.
\end{align*}
The remaining proof is split into two parts. In part $(I),$ we show that if $U\subset [1/4,3/4]^d,$ then, for any function $g\in C_d^\beta(U, K)$ and any integers $m \geq 1$ and $N \geq  5^d\vee (\beta+1)^d \vee (K+1)e^d,$ there exists a network $\widetilde g \in \mF\big(L, \big(d \sim 6(d+\lceil \beta\rceil)N \sim 1\big), s\big)$ with depth $L=8+(m+5)(1+\lceil \log_2 (d\vee \beta) \rceil)$ and number of parameters $s\leq 141 (d+\beta+1)^{3+d} N (m+6),$ such that 
\begin{align*}
	\| \widetilde g - g\|_{L^\infty(U)}\leq  (2K+1)(1+d^2+\beta^2) 6^d N2^{-m}+ K3^{2\beta} N^{-\frac{\beta}d}.
\end{align*}
In part $(II)$ of the proof, we discuss the general case.

{\it Part (I):} As the proof follows from a modification of the proof for Theorem in \cite{SH2017}, we only describe the differences using the notation in that article. The strategy in that paper is to define the set of grid points $\bD(M) :=\{\bx_{\bell} = (\ell_j/M)_{j=1,\ldots,r} : \bell =(\ell_1,\ldots, \ell_r) \in \{0,1,\ldots, M\}^r\}$ and to build a network that approximates the $\lfloor \beta\rfloor$-th order Taylor polynomial on each of this grid points. Denote by $P_{\bx_{\bell}}^\beta g (\bx)$ the $\lfloor \beta\rfloor$-th order Taylor polynomial around $\bx_\ell.$ Recall that $g$ is a H\"older function defined on $U\subset [1/4,3/4]^d.$ Thus, $P_{\bx_{\bell}}^\beta g (\bx)$ only exists if $\bx_{\bell} \in U.$ If $\bx_{\ell}\in [0,1]^d\setminus U,$ there exists a (not necessarily unique) grid point $\bz^* \in \argmin_{\bz \in U\cap \bD(M)} |\bx_\ell - \bz|_\infty.$ We then set $P_{\bx_{\bell}}^\beta g (\bx):= P_{\bz^*}^\beta g (\bx)$ and define 
\begin{align*}
	P^\beta g (\bx) := \sum_{\bx_{\bell} \in \bD(M)} P_{\bx_{\bell}}^\beta g (\bx) \prod_{j=1}^r (1- M |x_j-x^{\bell}_j|)_+.
\end{align*}
By adapting Lemma B.1  in \cite{SH2017}, we have for $1/M \leq 1/4$ and for any $\bx\in U,$
\begin{align}
	\big | P^\beta g (\bx) - g(\bx) \big| 
	&\leq \max_{\bx_{\bell} \in \bD(M): |\bx- \bx_{\bell}|_\infty \leq 1/M} \big| P_{\bx_{\bell}}^\beta g (\bx) -  g(\bx) \big| \notag \\
	&\leq \max_{\bz \in U: |\bx- \bz|_\infty \leq 3/M} \big| P_{\bz}^\beta g (\bx) -  g(\bx) \big| \label{eq.a1} \\
	&\leq K3^\beta M^{-\beta}. \notag
\end{align}
In Theorem 5 of \cite{SH2017}, we can modify the network $Q_1$ such that for $\bx_\ell \in \bD(M),$ (36) still holds, that is,
\begin{align*}
		\Big | Q_1(\bx) - \Big( \frac{P_{\bx_{\bell}}^\beta g (\bx)}{B} + \frac 12\Big)_{\bx_{\bell} \in \bD(M)} \Big|_\infty \leq \beta^2 2^{-m}.
\end{align*}
One should notice that all network parameters can be chosen to be bounded in absolute value by one and the construction does not require to enlarge the network architecture. To conclude the proof, one has to apply \eqref{eq.a1} which means that $K$ in the original bound has to be replaced by $K3^\beta$ in the step where Lemma B.1 is applied. Together with $(M+1)^d \leq N$ and $M\geq 4,$ the additional requirement $N\geq  5^d$ occurs. 

{\it Part (II):} Introduce the affine transformation $T: \R^d \rightarrow \R^d,$ $T\bx = R^{-1}\bx +(1/2, \ldots, 1/2)^\top.$ Define $U'=T(U)$ and observe that $U' \subseteq [1/4,3/4]^d.$ It is straightforward to see that if $g\in \mC_d^\beta(U, K),$ then, $h := g(T^{-1} \cdot ) \in \mC_d^\beta(U',R^\beta K).$ We can now apply the result from the first part, with $U$ replaced by $U'$ and $K$ replaced by $R^\beta K.$ This shows that for any integers $m \geq 1$ and $N \geq  5^d\vee (\beta+1)^d \vee (K+1)e^d,$ there exists a network $\widetilde h \in \mF\big(L, \big(d\sim  6(d+\lceil \beta\rceil)N\sim 1\big), s\big)$ with depth $L=8+(m+5)(1+\lceil \log_2 (d\vee \beta) \rceil)$ and number of parameters $s\leq 141 (d+\beta+1)^{3+d} N (m+6),$ such that 
\begin{align*}
	\| \widetilde h - h\|_{L^\infty(U')}\leq  (2R^\beta K+1)(1+d^2+\beta^2) 6^d N2^{-m}+ R^\beta  K3^{2\beta} N^{-\frac{\beta}d}.
\end{align*}
We now define the neural network $\widetilde g = \widetilde h \circ \sigma(T \cdot).$ Since $T$ is an affine transformation, this can be realized by adding one hidden layer to the network architecture of $\widetilde h.$ It also adds $2d$ non-zero parameters. Since $T(U) \subseteq [0,1]^d,$ we have $\widetilde g(\bx)= \widetilde h(T\bx)$ for all $\bx\in U.$ Together this shows that for any integers $m \geq 1$ and $N \geq  5^d\vee (\beta+1)^d \vee (K+1)e^d,$ there exists a network $\widetilde h \in \mF\big(L, \big(d\sim 6(d+\lceil \beta\rceil)N \sim 1\big), s\big)$ with depth $L=9+(m+5)(1+\lceil \log_2 (d\vee \beta) \rceil)$ and number of parameters $s\leq 142 (d+\beta+1)^{3+d} N (m+6),$ such that 
\begin{align*}
	\| \widetilde g - g\|_{L^\infty(U)} 
	&\leq \| \widetilde h - h\|_{L^\infty(U')}
	\leq  (2R^\beta K+1)(1+d^2+\beta^2) 6^d N2^{-m}+ R^\beta  K3^{2\beta} N^{-\frac{\beta}d}.
\end{align*}
\end{proof}

\begin{lem}[Lemma (A.1) in \cite{SH2017}]
\label{lem.mult}
For any positive integer $m,$ there exists a network $\Mult_m \in \mF(m+4,(2\sim 6\sim 1)),$ such that $\Mult_m(x,y) \in [0,1],$
\begin{align*}
	\big|\Mult_m (x,y) - x y \big| \leq 2^{-m}, \quad \text{for all} \ x,y \in [0,1],
\end{align*}
and $\Mult_m(0,y)=\Mult(x,0) =0.$
\end{lem}

\begin{lem}
\label{lem.V_j_embed}
If $\psi_j^{-1} \in C_{d^*}^1(\psi_j(V_j)), \psi_j \in C_d^1(V_j),$ then, for any $\delta>0,$ there exists a  $\delta'>0,$ such that
\begin{align*}
	\big(\psi_j(V_j^{-\delta}) \big)^{\delta'} := \big\{\by \in \R^{d^*}:\big|\by -\psi_j(V_j^{-\delta}) \big |_\infty \leq \delta'\big\}  \subseteq \psi_j(V_j), \quad \text{for all} \ j=1, \ldots, r. 
\end{align*}
\end{lem}

\begin{proof}
Recall that if $f\in C_r^1(U)$ with $U$ bounded, then $f$ is Lipschitz. Since $V_j\subset \mM$ and $\mM$ is compact, $V_j$ is a bounded set. Together with $\psi_j \in C_d^1(V_j),$ this shows that $\psi_j(V_j)$ is a bounded set and therefore also $\psi_j^{-1}$ is Lipschitz on $\psi_j(V_j).$ Thus, there exists a constant $L$ such that $|\psi_j^{-1}(\bu)-\psi_j^{-1}(\bv)|_\infty \leq L|\bu - \bv|_\infty.$ This also implies that $|\psi_j(\bx)-\psi_j(\by)|_\infty \geq L^{-1} |\bx-\by|_\infty$ for all $\bx, \by \in V_j.$

For any $\bu \in \psi_j(V_j)\setminus \psi_j(V_j^{-\delta/2} )$ and $\bv \in \psi_j(V_j^{-\delta})$ we have that $|\bv -\bu|_\infty\geq \delta/(2L)=:R.$ Suppose that there exist points $\bw \notin \psi_j(V_j)$ and $\bv \in \psi_j(V_j^{-\delta})$ with $|\bw - \bv|_\infty \leq R/2 .$ The set of points on the line $[0,1] \ni  t \mapsto t\bw +(1-t)\bv$ intersected with $\psi_j(V_j)\setminus \psi_j(V_j^{-\delta/2})$ cannot be empty and each such element must have a smaller $|\cdot |_\infty$-norm than $R/2$ which is a contradiction. This shows that $\bw \in  \psi_j(V_j)$ and yields the result for $\delta' =R/2>0.$
\end{proof}

\begin{lem}
\label{lem.comp_approx}
Let $K, \eps >0,$ and assume that $h_0,\widetilde h_0: U \subset \R^p \rightarrow V\subset \R^q$ such that $\|h_0-\widetilde h_0\|_{L^\infty(U)}\leq \eps.$ Let $V^\eps:=\{\bx \in \R^q:|\bx-V|_\infty\leq \eps\}.$ If $\widetilde h_1 :V^\eps \rightarrow [-K,K]$ and $h_1 \in \mC_{d^*}^\beta(V^\eps ,K),$ then,
\begin{align*}
	\big\| h_1 \circ h_0 - \widetilde h_1 \circ \widetilde h_0 \big\|_{L^\infty(U)}
	\leq K \|h_0-\widetilde h_0\|_{L^\infty(U)}^{\beta \wedge 1}
	+ \big\|h_1 - \widetilde h_1 \big\|_{L^\infty(V^\eps)}.
\end{align*}
\end{lem}

\begin{proof}
The inequality follows from $\| h_1 \circ h_0 - \widetilde h_1 \circ \widetilde h_0 \|_{L^\infty(U)}\leq \| h_1 \circ h_0 - h_1 \circ \widetilde h_0 \big\|_{L^\infty(U)}+\| h_1 \circ \widetilde h_0 - \widetilde h_1 \circ \widetilde h_0 \big\|_{L^\infty(U)}$ and $h_1 \in \mC_{d^*}^\beta(V^\eps,K).$ 
\end{proof}

\begin{proof}[Proof of Theorem \ref{thm.main_approx}]
We use the same notation as before and denote by $(V_1,\psi_1), \ldots(V_r,\psi_r)$ the charts. Since $\psi_j \in C_d^1(V_j),$ $\psi_j$ is Lipschitz and $\psi_j(V_j)$ is a bounded set. As we can always add a vector to the local coordinate map $\psi_j$ without changing the properties, we can (and will) assume that $\psi_j(V_j)\subset [1,\infty)^{d^*}.$

By Lemma \ref{lem.partition_of_unity}, there exist $\underline \delta>0$ and non-negative functions $\tau_j: \mM \rightarrow \R,$ $j=1, \ldots,r,$ such that for any $\gamma>0,$ and any  $\bx \in \mM,$ $\{\by\in \mM: \tau_j(\by)>0\}\subseteq V_j^{-\underline{\delta}},$  $\tau_j \in C_d^\gamma(\mM),$ and $\sum_{j=1}^r \tau_j(\bx) =1.$

We first show how to build networks approximating the coordinate maps $\psi_j,$ the functions $f\circ \psi_j^{-1}$ and the functions $\tau_j.$ We then merge these networks into a bigger network imitating the left hand side in \eqref{eq.f_split}, see also the schematic representation of the construction in Figure \ref{fig.manifold}.

By Definition \ref{defi.smooth_maps}, $\psi_j \in C_d^{\gamma_1}(V_j),$ where $\gamma_1 := (\beta \vee 1)d/d^*.$ To construct a neural network approximating $\psi_j$ on $V_j,$ we apply Theorem \ref{thm.approx_network_one_fct_neighborhood}. There exist positive constants $C_1:= C_1(d, d^*,\beta,  j, V_j, \delta')$ and $C_1':= C_1'(d, d^*,\beta,  j, V_j),$ such that for any integers $N_1 \geq  C_1, m_1 \geq ((\beta\vee 1)/d^*+1)\log_2 N_1,$ and $p_1^{\max}\geq 6d^*(d+\lceil \gamma_1\rceil)N_1,$  there exists a network $$\widetilde \psi_j \in \mF\big(L_1, (d\sim p_1^{\max} \sim d^*\big), s_1\big)$$ with depth $L_1=9+(m_1+5)(1+\lceil \log_2 (d\vee \gamma_1) \rceil)$ and number of parameters $s_1\leq 142 d^* (d+\gamma_1+1)^{3+d} N_1 (m_1+6),$ satisfying
\begin{align*}
	\| \widetilde \psi_j - \psi_j\|_{L^\infty(V_j)}\leq  C_1' N_1^{-\frac{\beta \vee 1}{d^*}} < \frac{\delta'}4.
\end{align*}
Given $0<\eta\leq 1/2,$  set $N_1 =  \lceil C_1 \vee (C_1')^{\frac{d^*}{\beta \vee 1}} (4r/\eta)^{\frac{d^*}{\beta}} \rceil.$ Then, there exist positive constants $K_1,K_1',K_1''$ that do not depend on $\eta,$ such that for any $L_1 \geq K_1 \log(1/\eta),$ any $p_1^{\max}\geq  K_1'\eta^{-d^*/\beta}$ and any $s_1\geq K_1'' L_1\eta^{-d^*/\beta},$ 
\begin{align}
	\inf_{\widetilde \psi_j\in \mF(L_1, (d\sim p_1^{\max}\sim d^*), s_1)} \, \| \widetilde \psi_j - \psi_j\|_{L^\infty(V_j)}\leq \Big(\frac{\eta}{4r}\Big)^{\frac{1}{\beta \wedge 1}} \wedge \frac{\delta'}4.
	\label{eq.NW1}
\end{align}

In the next step we construct a network approximating $f \circ \psi_j^{-1}.$ By Lemma \ref{lem.composition} (using that $\psi_j(V_j)$ is bounded), we have that $f \circ \psi_j^{-1} \in C_{d^*}^\beta(\psi_j(V_j)).$ Combined with Lemma \ref{lem.V_j_embed},  this also shows that there exists $\delta'>0,$ such that $f \circ \psi_j^{-1} \in C_{d^*}^\beta(\psi_j(V_j^{-\underline \delta})^{\delta'}).$ Using Theorem \ref{thm.approx_network_one_fct_neighborhood}, there exist constants $C_2:=C_2(d^*,\beta, \delta', \psi_j(V_j^{-\delta}), K')$ and $C_2':=C_2'(d^*,\beta, \psi_j(V_j^{-\delta})^{\delta'}, K'),$ such that for any integers $N_2 \geq C_2, m_2 \geq  (\beta/d^*+1)\log_2 N_2,$ and $p_2^{\max}\geq 6(d^*+\lceil \beta\rceil)N_2,$ there exists a network $$\overline{f \circ \psi_j^{-1}}  \in \mF\big(L_2, \big(d^*\sim  p_2^{\max} \sim 1\big), s_2\big)$$ with depth $L_2=9+(m_2+5)(1+\lceil \log_2 (d^* \vee \beta) \rceil)$ and number of parameters $s_2\leq 142 (d^*+\beta+1)^{3+{d^*}} N_2 (m_2+6),$ such that 
\begin{align*}
	\| \overline{f \circ \psi_j^{-1}} - f \circ \psi_j^{-1} \|_{L^\infty(\psi_j(V_j^{-\delta} )^{\delta'})}\leq  C_2' N_2^{-\frac{\beta}{d^*}}.
\end{align*}
From $\overline{f \circ \psi_j^{-1}}$ we construct now a deep ReLU network $\widetilde{f \circ \psi_j^{-1}}$ with two output units computing $\overline{f \circ \psi_j^{-1}}/(1+C_2' N^{-\frac{\beta}{d^*}})$ and $-\overline{f \circ \psi_j^{-1}}/(1+C_2' N^{-\frac{\beta}{d^*}}).$ This network is then in the class $\mF\big(L_2, \big(d^*\sim p_2^{\max} \sim 2\big), 2s_2\big).$ Since $f$ maps to $[-1,1],$ the network output of $\widetilde{f \circ \psi_j^{-1}}$ is in $[-1,1]^2.$ Also the output will approximate $(f \circ \psi_j^{-1},-f \circ \psi_j^{-1})$ up to an error $2C_2' N^{-\frac{\beta}{d^*}}.$

We can argue as above to find network architectures that lead to approximation error $\eta/4r.$  Set $N_2 =  \lceil C_2 \vee (8rC_2'/\eta)^{\frac{d^*}{\beta}} \rceil.$ Then, there exist positive constants $K_2,K_2',K_2''$ that do not depend on $\eta,$ such that for any $L_2 \geq K_2 \log(1/\eta),$ any $p_2^{\max}\geq  K_2'\eta^{-d^*/\beta}$ and any $s_2\geq K_2'' L_2\eta^{-d^*/\beta},$ 
\begin{align}
	\inf_{\widetilde{f \circ \psi_j^{-1}} \in \mF(L_2, (d^*\sim p_2^{\max}\sim 2), 2s_2)} \, \Big\| \widetilde{f \circ \psi_j^{-1}} - (f \circ \psi_j^{-1},-f \circ \psi_j^{-1}) \Big\|_{L^\infty(\psi_j(V_j^{\delta/2} \cap \mM)^{\delta'/2})} \leq \frac{\eta}{4r}.
	\label{eq.NW2}
\end{align}

Now, we build a deep network approximating $\tau_j \in \mC_d^{\beta d/d^*}(\mM, K_2)$ on $\mM.$ For that, we again apply Theorem \ref{thm.approx_network_one_fct_neighborhood}.  Write $\gamma_3:=\beta d/d^*.$ This shows existence of positive constants $C_3:= C_3(d, d^*,\beta, \mM)$ and $C_3':= C_3'(d, d^*,\beta, \mM),$ such that for any integers $N_3 \geq  C_3, m_3 \geq (\beta/d^*+1)\log_2 N_3,$ and $p_3^{\max}\geq 6 (d+\lceil \gamma_3\rceil)N_3,$  there is a network $$\overline \tau_j \in \mF\big(L_3, (d\sim p_3^{\max}\sim 1\big), s_3\big)$$ with depth $L_3=9+(m_3+5)(1+\lceil \log_2 (d\vee \gamma_3) \rceil)$ and number of parameters $s_1\leq 142 d^* (d+\gamma_3+1)^{3+d} N_3 (m_3+6),$ satisfying
\begin{align*}
	\| \overline \tau_j - \tau_j\|_{L^\infty(\mM)}\leq  C_3' N_3^{-\frac{\beta}{d^*}} <1.
\end{align*}
By adding one layer and two non-zero network parameters, we can also compute the network function $\widetilde \tau_j=(\overline \tau_j-C_3' N_3^{-\frac{\beta}{d^*}})_+.$ This means that $\widetilde \tau_j \in \mF\big(L_3+1, (d\sim p_3^{\max}\sim 1), s_3+2\big)$ and
\begin{align}
	\| \widetilde \tau_j - \tau_j\|_{L^\infty(\mM)}\leq  2C_3' N_3^{-\frac{\beta}{d^*}}.
	\label{eq.tau_j_bd}
\end{align}
Moreover, on $\mM,$ we have the property that the output of $\widetilde \tau_j$ is in $[0,1]$ and that the support of $\bx\mapsto \widetilde \tau_j(\bx)$ is contained in the support of $\tau_j.$

Set $N_3 = \lceil C_3 \vee (8rC_3'/\eta)^{\frac{d^*}{\beta}} \rceil.$ Then, there exist positive constants $K_3,K_3',K_3''$ that do not depend on $\eta,$ such that for any $L_3 \geq K_3 \log(1/\eta),$ any $p_3^{\max}\geq  K_3'\eta^{-d^*/\beta}$ and any $s_3\geq K_3'' L_3\eta^{-d^*/\beta},$ 
\begin{align}
	\inf_{\widetilde \tau_j \in \mF(L_3+1, (d\sim p_3^{\max}\sim 1), s_3+2)} \, \| \widetilde \tau_j - \tau_j \|_{L^\infty(\mM)}\leq \frac{\eta}{4r}.
	\label{eq.NW3}
\end{align}

In a next step, we combine the individual networks constructed so far in order to approximate $\tau_j (f\circ \psi_j^{-1}) \circ \psi_j$ for any $j=1, \ldots, r.$ 

First, we use the composition property \eqref{eq.composition_general}. Recall that $\psi_j(V_j) \subset [1,\infty)^{d^*}.$ We obtain that for any $L_{12} \geq (K_1+K_2)\log(1/\eta),$ any $p_{12}^{\max} \geq (K_1' \vee K_2') \eta^{-d^*/\beta}$ and any $s_{12}\geq 2(K_1''\vee K_2'') L_{12}\eta^{-d^*/\beta},$ there exists a network $\widetilde{f \circ \psi_j^{-1}} \circ \widetilde \psi_j = \widetilde{f \circ \psi_j^{-1}} \circ \sigma(\widetilde \psi_j) \in \mF(L_{12}+1, (d\sim p_{12}^{\max}\sim 2), s_{12})$ such that both \eqref{eq.NW1} and \eqref{eq.NW2} hold.  Using Lemma \ref{lem.comp_approx} with $\eps =\delta'/2$ together with \eqref{eq.NW1} and $\|\widetilde{f\circ \psi_j^{-1}}\|_{L^\infty(\psi_j(V_j^{-\delta} )^{\delta'})} \leq 1$ for the first inequality and \eqref{eq.NW1} and \eqref{eq.NW2} for the second inequality gives
\begin{align}
	&\Big \| \big(\widetilde{f\circ \psi_j^{-1}}\big) \circ \widetilde{\psi_j}
	- (f, -f)\Big\|_{L^\infty(V_j^{-\delta})} \notag \\
	&\leq 
	\big \| \widetilde{\psi_j}
	- \psi_j \big\|_{L^\infty(V_j)}^{\beta \wedge 1}
	+\big \|\widetilde{f\circ \psi_j^{-1}}
	- \big(f\circ \psi_j^{-1}, - f\circ \psi_j^{-1} \big)\big\|_{L^\infty((V_j^{-\delta})^{\delta'})} \label{eq.NW12}\\
	&\leq \frac{\eta}{2r}. \notag 
\end{align}

In a next step, we synchronize the depth using \eqref{eq.add_layers}. Thus, there exists a deep ReLU network $E_j$ with three outputs computing $(\widetilde{f \circ \psi_j^{-1}} \circ \widetilde \psi_j, - \widetilde{f \circ \psi_j^{-1}} \circ \widetilde \psi_j, \widetilde \tau_j )$ and 
\begin{align*}
	E_j \in \mF\Big( 1+L_{12}\vee  L_3, (d\sim p_{123}^{\max} \sim 3), s_{12}+s_3+2+ d(L_{12}\vee  L_3) \Big),
\end{align*}
with $p_{123}^{\max}:=p_3^{\max}+p_{12}^{\max}.$

For any positive integer $m,$ there exists by Lemma \ref{lem.mult} a network $\Mult_m \in \mF(m+4,(2\sim 6\sim 1)),$ such that $\Mult_m(x,y) \in [0,1],$ $|\Mult_m (x,y) - x y \big| \leq 2^{-m},$ for all $x,y \in [0,1],$ and $\Mult_m(0,y)=\Mult(x,0) =0.$ We can therefore also construct a neural network $\Mult^* \in \mF(\lceil \log_2(r/\eta)\rceil+6,(3\sim 12\sim 2)),$ that takes input $(x,y,z)$ and outputs $(\Mult_{m^*}(x,z),\Mult_{m^*}(y,z))$ with $m^*= \lceil \log_2(r/\eta)\rceil+2.$ In particular, 
\begin{align}
	\big|\Mult_{m^*}\big(x,z\big)-\Mult_{m^*}\big((-x)_+,z\big ) - xz\big| \leq \frac{\eta}{4r}
	\label{eq.mult_does_mult}
\end{align}
and $\Mult^*(x,y,0)=(0,0).$

The composed network $M_j:=\Mult^* \circ \sigma(E_j)$ therefore computes approximately $(\tau_j\cdot ((f\circ \psi_j^{-1}) \circ \psi_j)_+, \tau_j \cdot (- (f\circ \psi_j^{-1}) \circ \psi_j)_+ = (\tau_j  \cdot (f)_+,\tau_j \cdot (-f)_+).$ Using the parallelization rule, we can now build $r$ networks in parallel computing $(M_1, \ldots, M_r).$ The $2r$ outputs of this network are by construction of $\Mult^*$ non-negative. Denote the two outputs of $M_j$ by $M_{j1}$ and $M_{j2}.$ By adding one layer computing a weighted sum of all the outputs, we have constructed the network
 $$\widetilde{f}:=\sum_{\ell=1}^2 \sum_{j=1}^r (-1)^{\ell+1}\sigma(M_{j\ell})=\sum_{\ell=1}^2 \sum_{j=1}^r (-1)^{\ell+1} M_{j\ell}.$$ 
Moreover, there exist positive constants $c,C,C',$ such that for any $L\geq c\log(1/\eta),$ any $p\geq C\eta^{-d^*/\beta}$ and any $s\geq C' L \eta^{-d^*/\beta},$ $\widetilde{f} \in \mF(L,(d\sim p \sim 1),s).$

It remains to bound the approximation error of the network $\widetilde f.$ For the estimate, we use in the first step that due to Lemma \ref{lem.partition_of_unity} and the construction of $\widetilde \tau_j$, for any $j,$ $\tau_j$ and $\widetilde \tau_j$ vanish outside the set $V_j^{-\delta}$ and $\Mult^*(x,y,0)=(0,0).$ The second inequality follows from \eqref{eq.mult_does_mult}. For the third inequality, recall that $\widetilde \tau_j \leq 1$ and $\|f\|_{L^\infty(\R^d)}\leq 1.$ Together with \eqref{eq.NW3} and \eqref{eq.NW12}, this yields 
\begin{align*}
	\big\| \widetilde{f} - f\big\|_{L^\infty(\mM)}
	&\leq \sum_{j=1}^r \Big \|  (M_{j1}-M_{j2})
	- \tau_j f\Big\|_{L^\infty(V_j^{-\delta})} \\
	&\leq \frac{\eta}{4} + \sum_{j=1}^r \Big \| \widetilde \tau_j \big(\widetilde{f\circ \psi_j^{-1}}\big) \circ \widetilde{\psi_j}
	- \tau_j f\Big\|_{L^\infty(V_j^{-\delta})} \\
	&\leq \frac{\eta}{4} + \sum_{j=1}^r \big \| \widetilde \tau_j
	- \tau_j \big\|_{L^\infty(V_j)} 
	+\Big \| \big(\widetilde{f\circ \psi_j^{-1}}\big) \circ \widetilde{\psi_j}
	- f\Big\|_{L^\infty(V_j^{-\delta})} \\
	&\leq \eta
\end{align*}
completing the proof.
\end{proof}

\bibliographystyle{acm}       
\bibliography{bibDL}           

\end{document}